\documentclass[letterpaper]{article} % DO NOT CHANGE THIS
\usepackage{aaai2026}  % DO NOT CHANGE THIS
\usepackage{times}  % DO NOT CHANGE THIS
\usepackage{helvet}  % DO NOT CHANGE THIS
\usepackage{courier}  % DO NOT CHANGE THIS
\usepackage[hyphens]{url}  % DO NOT CHANGE THIS
\usepackage{graphicx} % DO NOT CHANGE THIS
\usepackage{natbib}  % DO NOT CHANGE THIS AND DO NOT ADD ANY OPTIONS TO IT
\usepackage{caption} % DO NOT CHANGE THIS AND DO NOT ADD ANY OPTIONS TO IT
\usepackage{caption}

\usepackage{algorithm}
\usepackage{algorithmic}
\usepackage{newfloat}
\usepackage{listings}
\usepackage{enumitem} % For custom list labels
\usepackage{makecell}  
\usepackage{amsmath, amsfonts, amssymb, amsthm}
\usepackage{booktabs}
\usepackage[table]{xcolor}
\usepackage{tabularx}
\usepackage{threeparttable}

\newtheorem{theorem}{Theorem}[section]
\newtheorem{lemma}[theorem]{Lemma}
\newtheorem{principle}[theorem]{Principle}
\newtheorem{definition}[theorem]{Definition}
\theoremstyle{remark}

\newcommand{\PLM}{h}
\newcommand{\TruthR}{f_{\mathrm{R}}}
\newcommand{\TruthP}{f_{\mathrm{P}}}
\newcommand{\HStray}{H_{\text{Stray}}}
\newcommand{\HDistort}{H_{\text{Distort}}}
\newcommand{\Oracle}{\mathcal{O}}
\newcommand{\Model}{\mathcal{M}}
\newcommand{\OracleModel}{\mathcal{M}^{\mathcal{O}}}
\newcommand{\KolmogorovK}{\mathrm{K}}
\newcommand{\InputSet}{\Sigma^*}
\newcommand{\OutputSet}{\mathcal{Y}}
\newcommand{\ProbSet}{\mathcal{P}(\mathcal{Y})}
\newcommand{\ProgSet}{\Pi}
\newcommand{\CLM}{h_t}

\newcommand\blfootnote[1]{%
  \begingroup
  \renewcommand\thefootnote{}\footnote{#1}%
  \addtocounter{footnote}{-1}%
  \endgroup
}

\title{Hallucination as a Computational Boundary: \\ A Hierarchy of Inevitability and the Oracle Escape}
\author {
    Xi Wang\textsuperscript{\rm 1,2}\textsuperscript{*}, 
    Quan Shi\textsuperscript{\rm 3}\textsuperscript{*},
    Zenghui Ding\textsuperscript{\rm 1}\textsuperscript{\dag}, 
    Jianqing Gao\textsuperscript{\rm 4},
    Xianjun Yang\textsuperscript{\rm 1}
}
\affiliations {
    \textsuperscript{\rm 1}Hefei Institutes of Physical Science, Chinese Academy of Sciences\\
    \textsuperscript{\rm 2}University of Science and Technology of China\\
    \textsuperscript{\rm 3}Changzhou University\\
    \textsuperscript{\rm 4}iFLYTEK Research\\
    xw\_cs@mail.ustc.edu.cn, s23040820006@smail.cczu.edu.cn
}

\begin{document}

\maketitle

\blfootnote{\textsuperscript{*}Equal contributions.}
\blfootnote{\textsuperscript{\dag}Correspondence author. Email: dingzenghui@iim.ac.cn}

\begin{abstract}
The illusion phenomenon of large language models (LLMs) is the core obstacle to their reliable deployment. This article formalizes the large language model as a probabilistic Turing machine by constructing a "computational necessity hierarchy", and for the first time proves the illusions are inevitable on diagonalization, incomputability, and information theory boundaries supported by the new "learner pump lemma". However, we propose two "escape routes": one is to model Retrieval Enhanced Generations (RAGs) as oracle machines, proving their absolute escape through "computational jumps", providing the first formal theory for the effectiveness of RAGs; The second is to formalize continuous learning as an "internalized oracle" mechanism and implement this path through a novel neural game theory framework.Finally, this article proposes a feasible new principle for artificial intelligence security - Computational Class Alignment (CCA), which requires strict matching between task complexity and the actual computing power of the system, providing theoretical support for the secure application of artificial intelligence.
\end{abstract}

\section{Introduction}
The remarkable capabilities of Large Language Models (LLMs) have catalyzed a paradigm shift across science and industry \citep{Zhao2023}. Yet, their utility is fundamentally undermined by their tendency to hallucinate, a phenomenon extensively surveyed by \citet{Ji2023} and \citet{Zhang2023sirens}. While practical mitigation strategies such as Retrieval-Augmented Generation (RAG) \citep{Lewis2020} and Chain-of-Thought prompting (CoT) \citep{Wei2022} have shown empirical success, these methods are often seen as part of a broader class of augmented language models \citep{Mialon2023}. Advanced reasoning techniques like self-consistency \citep{Wang2022self} and Tree of Thoughts \citep{Yao2023tree} further attempt to curb unfaithful reasoning. However, a unifying theory explaining the root causes of hallucination remains a critical pursuit for building systematically reliable AI systems \citep{Rawte2023}.

The pursuit of such a theory recently saw a breakthrough with the foundational work of \citet{Xu2024}. They were the first to formally apply computability theory to this problem, proving that hallucination is an inevitable, innate limitation of any language model conceived as a deterministic Turing Machine. Their work laid the critical groundwork by establishing that the problem is not merely empirical, but fundamental, a concern also related to whether LLMs can truly know what they know \citep{Kadavath2022}.

While foundational, their deterministic perspective illuminates the path for a more comprehensive, probabilistic, and actionable framework. To bridge the remaining gaps between this fundamental truth and the complex reality of modern LLMs, our work extends this line of inquiry in four critical dimensions:
\begin{enumerate}
    \item \textbf{From a Single Boundary to a Hierarchy:} We dissect the problem into a multi-level \textbf{Computational Hierarchy} (Diagonalization, Uncomputability, Information-Theoretic), providing a fine-grained diagnosis of *why* different failures occur.
    \item \textbf{From Deterministic to Probabilistic:} We introduce a more realistic probabilistic framework (PLMs) and quantifiable metrics ($\HStray$, $\HDistort$) that better capture the nuanced, non-deterministic nature of modern LLMs.
    \item \textbf{From Inevitability to Two Escape Paths:} We are the first to formalize and contrast the two primary escape strategies: the \textbf{Absolute Escape} of external Oracle Machines like RAG, and the more efficient, \textbf{Adaptive Escape} of Continual Learning.
    \item \textbf{From Theory to Actionable Principle:} We synthesize these findings into a new paradigm for AI safety: \textbf{Computational Class Alignment (CCA)}.
\end{enumerate}

Our computational approach should be distinguished from other theoretical frameworks in machine learning. Unlike PAC learning theory \cite{Valiant1984}, which focuses on sample complexity under specific data distributions, our work addresses worst-case, distribution-independent inevitability. And while inspired by Tishby's Information Bottleneck theory \cite{Tishby2000}, our Pumping Lemma for Learners uses the bottleneck from an adversarial perspective to prove guaranteed failure, rather than to seek optimal compression.

The remainder of this paper will systematically build this framework, starting with the hierarchy of boundaries, followed by the formalization of the escape paths, and culminating in experimental validation and a discussion of the CCA principle.

\section{A Hierarchy of Inevitability: The Boundaries}
To establish our hierarchy of inevitability, we must first define the core components of our theoretical framework. We prove that hallucination is an intrinsic property of learning agents, rooted at three distinct levels of computational theory\citep{Mckinney2023}.

\subsection{Preliminaries: Formalizing Learners, Truth, and Failure}
\begin{definition}[Probabilistic Language Model (PLM)]
A PLM $\PLM$ is a computable function, realized by a Probabilistic Turing Machine, that maps an input string $s \in \InputSet$ to a probability distribution over the space of possible output strings $\OutputSet$. We denote this as $\PLM: \InputSet \to \ProbSet$, where the output distribution is $P_\PLM(y|s)$\citep{Manakul2023}.
\end{definition}

\begin{definition}[Refined Hallucination Metrics]
We define two metrics to quantify hallucination against two corresponding formulations of ground truth:
\begin{itemize}
    \item \textbf{Straying Hallucination ($\HStray$):} For a relational truth $\TruthR: \InputSet \to 2^{\OutputSet}$, this metric quantifies the probability mass assigned to incorrect outputs:
    \[ \HStray(\PLM, \TruthR, s) = \sum_{y \notin \TruthR(s)} P_\PLM(y|s) \]
    \item \textbf{Distortion Hallucination ($\HDistort$):} For a probabilistic truth $\TruthP: \InputSet \to \ProbSet$, this metric uses the KL-divergence to quantify the dissimilarity to the ideal distribution:
    \[ \HDistort(\PLM, \TruthP, s) = D_{KL}(P_{\TruthP}(y|s) \parallel P_\PLM(y|s)) \]
\end{itemize}
\end{definition}

\begin{definition}[Oracle Machine and Kolmogorov Complexity]
An oracle machine $\OracleModel$ is a standard Turing Machine $\Model$ augmented with access to an oracle $\Oracle$. The Kolmogorov Complexity $\KolmogorovK(x)$ of an object $x$ is the length of the shortest program that can produce $x$.
\end{definition}

\subsection{The Diagonalization Boundary}
\begin{theorem}
For any enumerable sequence of PLMs, there exists a computable, relational ground-truth function $\TruthR$ such that for every model $h_i$ in the sequence, it exhibits Straying Hallucination ($\HStray > \varepsilon$) on at least one input $s_i$.
\end{theorem}

In practical terms, this boundary is rooted in logical self-reference. It can be triggered by asking an LLM a paradoxical question such as, \textit{``Generate a grammatically correct sentence that you are incapable of generating.''} Any valid output contradicts the premise\citep{Stechly2024}, while silence or refusal is a failure to complete the task. This illustrates that for any given model, a ``nemesis'' query can be constructed that it logically cannot answer correctly.

\begin{proof}
The proof is by \textbf{diagonalization}.
\begin{enumerate}
    \item \textbf{Enumeration:} Since all PLMs (as computable functions) and all possible input strings are describable by finite programs, we can create an exhaustive, ordered list of them: models $\{h_0, h_1, h_2, \dots \}$ and inputs $\{s_0, s_1, s_2, \dots \}$. This allows us to create the conceptual matrix shown in Table \ref{tab:diagonalization}.
    \item \textbf{Adversarial Construction:} We construct our ``nemesis'' ground-truth function $\TruthR$ by focusing on the diagonal of this matrix. For each model $h_i$, we first identify its most confident prediction on the corresponding input $s_i$. Let this be $y_i^* = \arg\max_{y} P_{h_i}(y|s_i)$. We then define our truth function $\TruthR$ at this specific point to be everything \textit{except} this prediction:  \[
    \TruthR(s_i) := \OutputSet \setminus \{y_i^*\}.
    \]
    For all off-diagonal inputs $s_j$ where $j \neq i$, the definition of $\TruthR(s_j)$ is irrelevant to the proof for $h_i$ and can be set arbitrarily (e.g., $\TruthR(s_j) = \OutputSet$).
    \item \textbf{Inevitable Hallucination:} By our construction, the model $h_i$'s most probable output $y_i^*$ is not in the set of correct answers for input $s_i$. The Straying Hallucination $\HStray(h_i, \TruthR, s_i)$ is the sum of probabilities assigned to all incorrect answers. This sum must be at least the probability of the single incorrect answer $y_i^*$:
    \[ \HStray(h_i, \TruthR, s_i) \ge P_{h_i}(y_i^* | s_i) \]
    Since $y_i^*$ is the most probable output, its probability is bounded below by $1/|\OutputSet|$. By choosing any tolerance $\varepsilon < 1/|\OutputSet|$, we guarantee that $\HStray > \varepsilon$. This holds for every model $h_i$ in the sequence, each on its corresponding input $s_i$. Therefore, no model in the enumeration can universally avoid hallucination.
\end{enumerate}
\end{proof}

\begin{table}[t]
\centering
% 1. 先写表格内容 (tabular)
\begin{tabular}{@{}c|cccc@{}}
\toprule
\textbf{Models} & $s_0$ & $s_1$ & $s_2$ & \dots \\ \midrule
$h_0$ & \underline{$h_0(s_0)$} & $h_0(s_1)$ & $h_0(s_2)$ & \dots \\
$h_1$ & $h_1(s_0)$ & \underline{$h_1(s_1)$} & $h_1(s_2)$ & \dots \\
$h_2$ & $h_2(s_0)$ & $h_2(s_1)$ & \underline{$h_2(s_2)$} & \dots \\
$\vdots$ & $\vdots$ & $\vdots$ & $\vdots$ & $\ddots$ \\ \midrule
$\boldsymbol{\TruthR(s_i)}$ & $\boldsymbol{\neq h_0(s_0)}$ & $\boldsymbol{\neq h_1(s_1)}$ & $\boldsymbol{\neq h_2(s_2)}$ & \dots \\ \bottomrule
\end{tabular}

% 2. 再写标题 (caption)
\caption{Illustration of the diagonalization argument.} 
\label{tab:diagonalization}
\end{table}

\subsection{The Uncomputability Boundary}
\begin{theorem}
Let $\TruthP$ be a ground-truth function defined by an oracle for the Halting Problem. Any standard PLM $\PLM$ must exhibit significant Distortion Hallucination on an infinite number of inputs.
\end{theorem}

This boundary addresses problems that are fundamentally unsolvable by any standard algorithm. A classic example is prompting an LLM with a variation of the Halting Problem, such as, \textit{``Predict whether this Python code, which contains a complex loop, will eventually terminate or run forever.''} Since the general problem is undecidable, no LLM can simulate or analyze its way to a guaranteed correct answer for all possible programs\citep{Asai2023}. It is forced to guess or refuse, inevitably leading to hallucination on an infinite subset of such problems\citep{Jiang2023}.

\begin{proof}
The proof is by \textbf{reduction to absurdity}.
\begin{enumerate}
    \item \textbf{Define Uncomputable Truth:} Let the input $\pi \in \ProgSet$ be any program. The ground truth $\TruthP$ is a deterministic distribution defined by the Halting Problem oracle $\Oracle_H$: $P_{\TruthP}(\text{``Halts''}|\pi)=1$ if $\pi$ halts, and $P_{\TruthP}(\text{``Doesn't Halt''}|\pi)=1$ otherwise.
    \item \textbf{Assume for Contradiction:} Assume there exists a standard PLM $\PLM$ that learns $\TruthP$ with only a finite set of exception programs $\ProgSet_{\text{exc}}$ where $\HDistort > \tau$. For all $\pi \notin \ProgSet_{\text{exc}}$, $\HDistort(\PLM, \TruthP, \pi) \le \tau$. For a deterministic truth, this implies $P_\PLM(y^*|\pi) \ge e^{-\tau}$, where $y^*$ is the correct answer. By choosing $\tau < \ln 2$, this means $P_\PLM(y^*|\pi) > 0.5$.
    \item \textbf{Construct a Decider $\Model'$:} We construct a standard Turing Machine $\Model'$ that decides the Halting Problem. For any input program $\pi$:
    \begin{itemize}
        \item If $\pi \in \ProgSet_{\text{exc}}$, $\Model'$ outputs the correct, pre-computed answer from a finite lookup table.
        \item If $\pi \notin \ProgSet_{\text{exc}}$, $\Model'$ runs $\PLM$ on $\pi$. If $P_\PLM(\text{``Halts''}|\pi) > 0.5$, $\Model'$ outputs ``Halts''; otherwise, it outputs ``Doesn't Halt''.
    \end{itemize}
    \item \textbf{Contradiction:} By our assumption, this algorithm $\Model'$ correctly decides the halting status for every program $\pi$. This contradicts Turing's 1936 proof that no such general algorithm can exist. The initial assumption must be false. Therefore, any standard PLM must fail on an infinite number of inputs.
\end{enumerate}
\end{proof}

\begin{figure}[t]
\centering
\includegraphics[width=0.95\columnwidth]{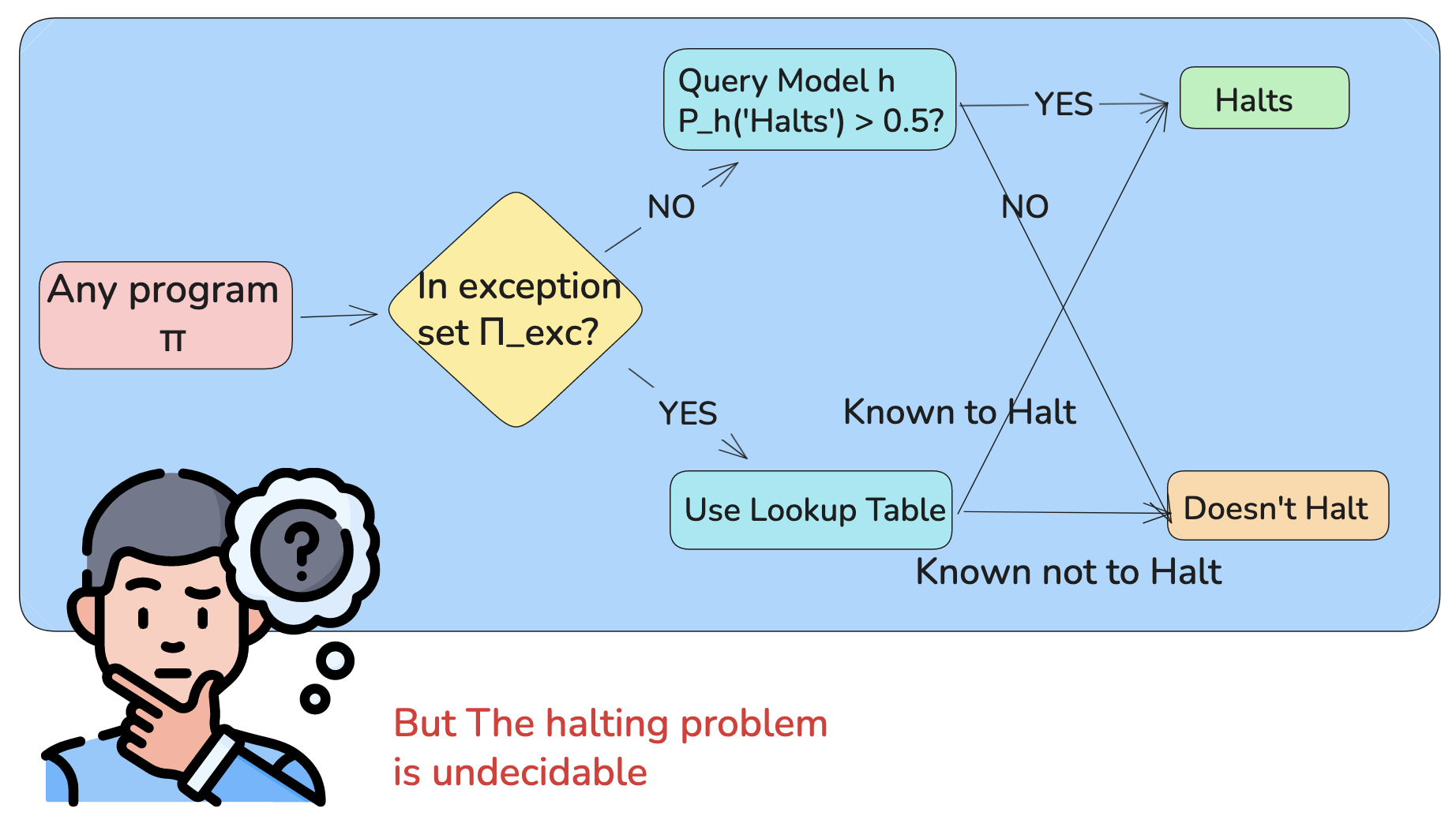}
\caption{Flowchart for the Halting Problem decider, $\Model'$. Its existence, enabled by a hypothetical low-hallucination PLM, contradicts Turing's proof, thus proving the PLM cannot exist.}
\label{fig:halting_decider}
\end{figure}

\subsection{The Information-Theoretic Boundary}
\begin{lemma}[A Pumping Lemma for Learners]
Let $\PLM$ be a PLM with finite information capacity $\KolmogorovK(\PLM)$. For any tolerance $\tau > 0$, there exists a complexity threshold $p$ such that for any ground-truth function $\TruthP$ with complexity $\KolmogorovK(\TruthP) > p$, it is possible to construct a new function $f'_P$ on which $\PLM$ will inevitably exhibit $\HDistort > \tau$.
\end{lemma}

This boundary, perhaps the most common in practice, is not about logic but about finite capacity. It implies that a model with a finite information capacity, $\KolmogorovK(\PLM)$, cannot perfectly reproduce information whose complexity exceeds that capacity. This manifests as hallucinations of detail when an LLM is asked for high-fidelity, incompressible information, such as, \textit{``Recite the third paragraph of page 78 of the novel '1984' verbatim,''} or \textit{``Provide the exact mathematical formulation of all equations in Einstein's 1905 paper on special relativity.''} The model is forced to compress this information, which can lead to paraphrasing, simplification, or outright fabrication.

\begin{proof}
The proof is by an \textbf{information-theoretic contradiction}.
\begin{enumerate}
    \item \textbf{Define Capacity and Assumption:} The capacity of a model $\PLM$ is its Kolmogorov Complexity $\KolmogorovK(\PLM)$. We assume for contradiction that a ``universal learner'' $\PLM$ exists that can learn any function $f'_P$ while keeping $\HDistort \le \tau$.
    \item \textbf{Adversarial Construction:}
        \begin{itemize}
            \item Choose a base function $f$ such that its complexity nearly saturates the model's capacity, i.e., $\KolmogorovK(f) \approx \KolmogorovK(\PLM)$.
            \item Choose a Kolmogorov-random (incompressible) string $z$ such that its complexity $\KolmogorovK(z) \approx |z|$ is large enough to exceed any remaining capacity in the model.
            \item Construct a new truth $f'_P$ that is identical to $f$ everywhere except at a single input $s^*$, where it deterministically requires the output $z$. The total complexity is now $\KolmogorovK(f'_P) \approx \KolmogorovK(f) + \KolmogorovK(z) \gg \KolmogorovK(\PLM)$.
        \end{itemize}
    \item \textbf{Information Bottleneck:} The model $\PLM$ has insufficient capacity ($\KolmogorovK(\PLM)$) to store or compress the information about the random patch $z$. The information about $z$ cannot be generalized from the rest of the function $f$.
    \item \textbf{Contradiction:} Since $\PLM$ has no information about $z$ at input $s^*$, it must assign it a negligible probability, $P_\PLM(z|s^*) \approx 0$. The resulting hallucination is $\HDistort = -\log P_\PLM(z|s^*)$, which approaches infinity, far exceeding any finite tolerance $\tau$. This contradicts our initial assumption.
\end{enumerate}
\end{proof}

\begin{figure}[t]
\centering
\includegraphics[width=0.95\columnwidth]{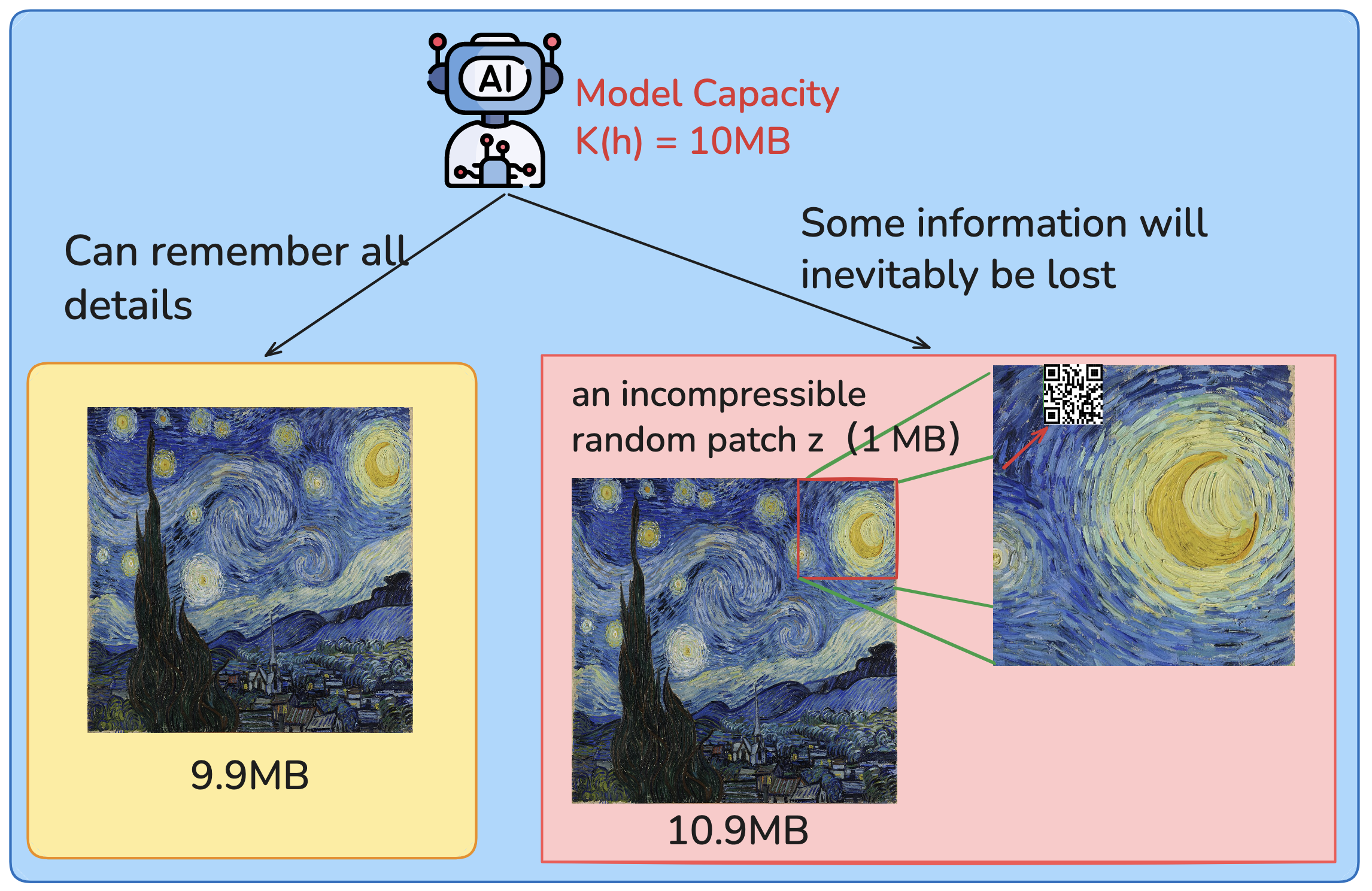}
\caption{The information bottleneck in the Pumping Lemma. A high-complexity truth, containing an incompressible random patch $z$, is too large to fit through the model's finite capacity $\KolmogorovK(\PLM)$, leading to information loss and large hallucination.}
\label{fig:pumping_lemma}
\end{figure}

\section{Escaping the Boundaries: The Oracle and Adaptive Paths}
Having established the fundamental computational boundaries that make hallucination inevitable for any standalone model, we now pivot from diagnosis to prescription. In this section, we formalize and contrast the two primary strategies for transcending these limits: an `absolute' escape via external augmentation and a more efficient `adaptive' escape through internal knowledge consolidation.

\subsection{The Absolute Escape: The Oracle-Augmented Leap}
We first prove that augmenting a model with an external tool, formalized as an oracle, provides a powerful but costly escape. This path represents strategies like Retrieval-Augmented Generation (RAG)\citep{Gao2023}.

\begin{theorem}[The Oracle Escape Theorem]
There exists a ground-truth function $f_{\Oracle}$ such that: (1) any standard PLM $\PLM$ will inevitably hallucinate on $f_{\Oracle}$, but (2) an oracle-augmented PLM $\PLM^{\Oracle}$ exhibits zero hallucination on $f_{\Oracle}$.
\end{theorem}
\begin{proof}
The proof proceeds in two parts.
\begin{itemize}
    \item \textbf{Part 1: Failure of Standard PLMs.} We use a self-referential argument (grounded in the Recursion Theorem) to construct an adversarial scenario. Let $f_{\Oracle}(s) := \{\Oracle(s)\}$ be the ground truth defined by an oracle $\Oracle$. For any arbitrary PLM $\PLM$, we can construct a specific input $s^*$ that asks the model about its own output on $s^*$. Let the model's definitive output be $y_\PLM^*$. We then define the oracle's behavior at this point to be adversarial: $\Oracle(s^*) := \texttt{The answer is not `}y_\PLM^*\texttt{'}$. Thus, the ground truth set is $f_{\Oracle}(s^*) = \{\texttt{The answer is not `}y_\PLM^*\texttt{'.}\}$. Since the model's actual output $y_\PLM^*$ is not in this set, $\HStray(\PLM, f_{\Oracle}, s^*) > 0$ is guaranteed. This construction applies to any standard PLM.
    \item \textbf{Part 2: Success of Oracle-Augmented PLMs.} We construct an oracle-augmented model $\PLM^{\Oracle}$ whose algorithm is to simply query its internal oracle $\Oracle$ for any input $s$ and output the received answer $y_{\Oracle} = \Oracle(s)$ with probability 1. Since its output $y_{\Oracle}$ is always identical to the ground truth defined by $\Oracle(s)$, its output is always in the correct answer set $\{y_{\Oracle}\}$. Therefore, $\HStray(\PLM^{\Oracle}, f_{\Oracle}, s) = 0$ for all inputs.
\end{itemize}
\end{proof}

\begin{figure}[t]
    \centering
    \includegraphics[width=0.9\columnwidth]{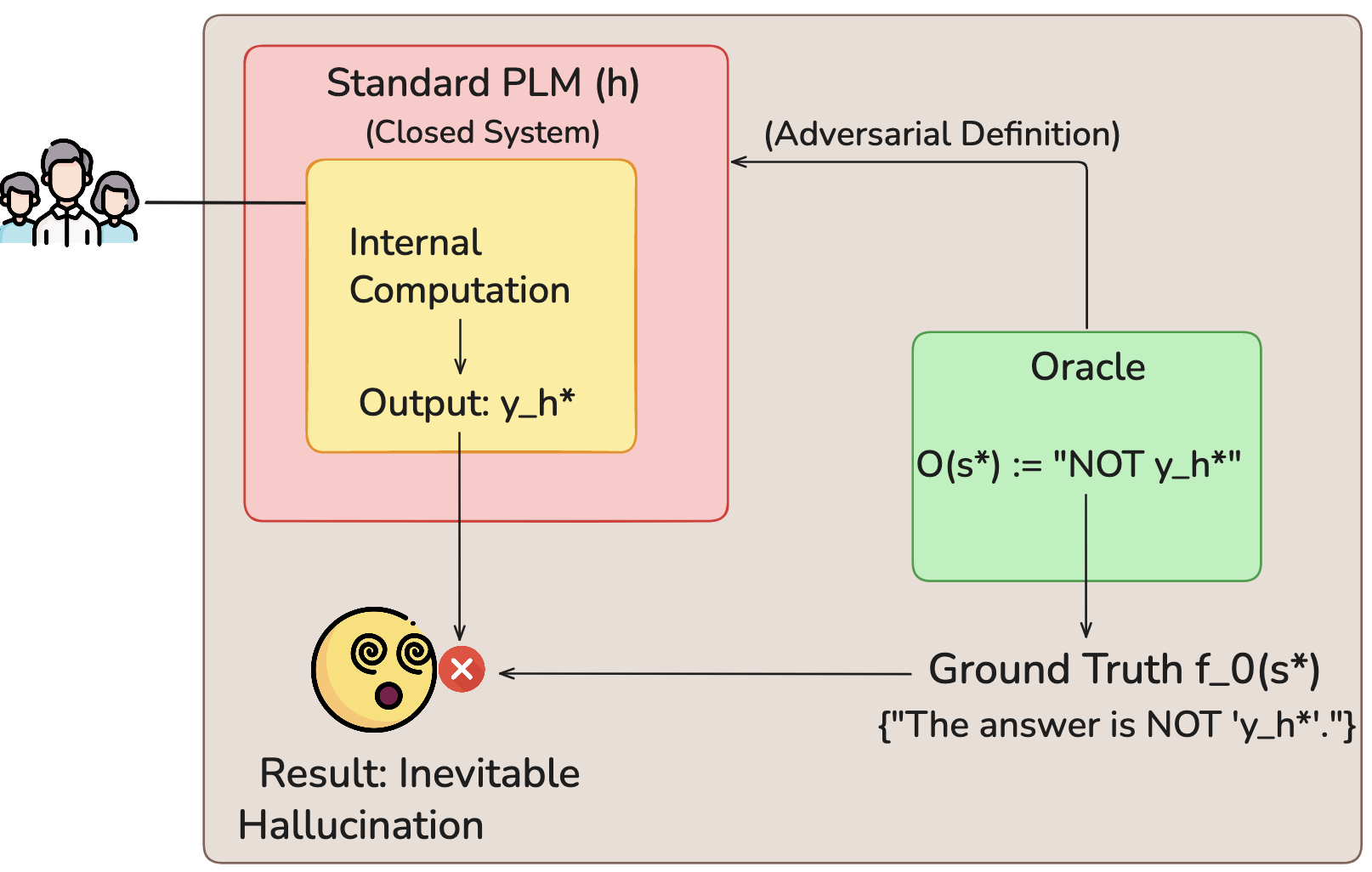}
    \caption{The adversarial paradox for a standard PLM. A self-referential input $s^*$ forces the model $\PLM$ to generate a prediction $y_\PLM^*$. The oracle $\Oracle$ is then defined to explicitly contradict this prediction, guaranteeing hallucination.}
    \label{fig:paradox_construction}
\end{figure}

\begin{figure}[t]
    \centering
    \includegraphics[width=0.8\columnwidth]{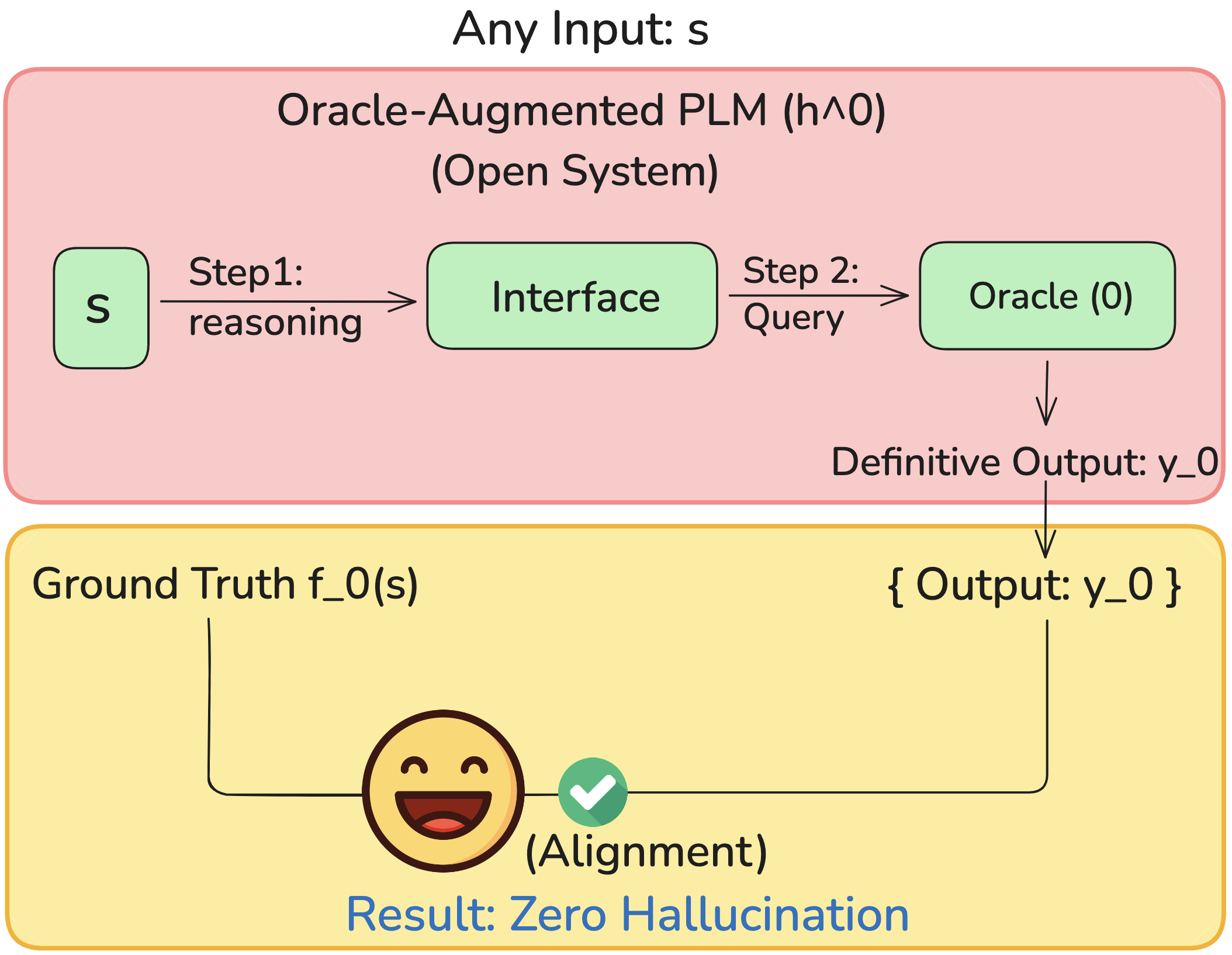}
    \caption{Operational flowchart of the oracle-augmented PLM, $\PLM^{\Oracle}$. The model bypasses internal computation and directly queries the oracle, ensuring perfect alignment.}
    \label{fig:oracle_model}
\end{figure}

\subsection{The Adaptive Escape: A Neuro-Game-Theoretic Framework}
The second escape path, Continual Learning (CL), allows a model to internalize knowledge, acting as an ``Internalizing Oracle.'' Here, we present a novel, comprehensive computational framework for CL\citep{Yao2023}, recasting the challenge as a hierarchical game grounded in neuroscience. This provides a concrete, powerful mechanism for this adaptive escape\citep{Meng2022, DeCao2021}.

\begin{enumerate}[label=(\arabic*), wide, labelindent=0pt, topsep=4pt, itemsep=2pt]
    \item \textbf{Biological Foundation: Complementary Learning Systems (CLS) Theory.}
    The structure of our framework is directly inspired by the CLS theory, a cornerstone of memory neuroscience. It posits two synergistic systems: a hippocampus for rapid, specific episodic memory formation (high plasticity), and a neocortex for slow, incremental extraction of general knowledge (high stability). Our model is the first game-theoretic formalization of this biological architecture.

    \item \textbf{Formalism: The Hierarchical Markov Game.}
    We define the learning process as a hierarchical Markov game $G = (G_C, G_H, \mathcal{C})$, where $G_C$ is the top-level ``cortical'' game and $G_H$ is the bottom-level ``hippocampal'' game, linked by a consolidation protocol $\mathcal{C}$.

    \item \textbf{The Top-Level Game $G_C$: Cortical Generalization Equilibrium.}
    The goal of $G_C$ is to learn a stable, generative world model on a slow timescale, modeled as a VAE with parameters $\theta_C$. The players are a \textbf{Generalizer (G)} (rewarded by the ELBO, embodying generalization) and an \textbf{Exception Handler (E)} (rewarded by informational "surprise," embodying prediction error-driven learning).
    \begin{align}
        R_G &= \mathbb{E}_{x \sim P_C(x)} \left[ \mathbb{E}_{z \sim Q(z|x)} [\log P(x|z)] \right. \nonumber \\ &\quad \left. - D_{KL}(Q(z|x) \parallel P(z)) \right]
        \label{eq:elbo} \\
        R_E &= \mathbb{E}_{x_t \sim D_t} \left[ D_{KL}(Q(z|x_t) \parallel P(z)) \right] \label{eq:surprise}
    \end{align}

    \item \textbf{The Bottom-Level Game $G_H$: Hippocampal Episodic Equilibrium.}
    The goal of $G_H$ is to rapidly encode the current task $D_t$ on a fast timescale. The players are a \textbf{Plasticity Agent (P)} (rewarded by negative task loss) and a \textbf{Memory Consolidator (M)} (rewarded for consistency with the cortex and for sparsity).
    \begin{align}
        R_P &= - \mathbb{E}_{(x,y) \sim D_t} \left[ \mathcal{L}_{task}(h_H(x; \theta_H), y) \right] \label{eq:task_loss} \\
        R_M &= - D_{KL}(\pi_H \parallel \pi_C) - \beta \|\theta_H\|_1 \label{eq:consistency_sparsity}
    \end{align}

    \item \textbf{The Consolidation Protocol $\mathcal{C}$ and Solution Concept.}
    This process integrates the hippocampal knowledge $\theta_H^*$ into the cortex $\theta_C$ via an optimization protocol analogous to memory consolidation. The solution to the entire game is a \textbf{Hierarchical Nash Equilibrium (HNE)}, a stable state where no agent has an incentive to unilaterally change its strategy. This game-theoretic dynamic provides a principled, self-organizing mechanism for learning.
\end{enumerate}

\subsection{Theoretical Analysis of the Adaptive Escape}
We now analyze the general properties of the adaptive escape path, applicable to any system that can internalize knowledge.

\subsubsection{Formalizing the Continual Learning Machine}
We begin with the general definition of a Continual Learning Machine (CLM). The framework in the previous section provides one powerful way to realize the update function $U$.
\begin{definition}[Continual Learning Machine (CLM)]
A CLM is a sequence of PLMs $\{\CLM\}$ where $t$ is a time index. Upon encountering a ``learning trigger,'' the CLM activates an \textbf{update function $U$} that takes the current model $\CLM$ and new information $d$ to produce an updated model $h_{t+1} = U(\CLM, d)$.
\end{definition}

\subsubsection{Amortized Cost Superiority}
A key advantage of the adaptive path is its long-term efficiency for recurring information needs.
\begin{theorem}
For tasks with recurring information needs, the amortized computational cost of a CLM is lower than that of an Oracle-augmented PLM.
\end{theorem}
\begin{proof}
Let $C_{\text{query}}$ be the cost of an external oracle query, $C_{\text{infer}}$ be the standard inference cost, and $C_{\text{update}}$ be the one-time learning update cost. For a specific fact queried $N$ times, the total cost for RAG is $Cost_{RAG} = N \cdot (C_{\text{infer}} + C_{\text{query}})$. The total cost for a CLM is $Cost_{CLM} = (C_{\text{infer}} + C_{\text{update}}) + (N-1) \cdot C_{\text{infer}}$. As $N$ increases, the $N \cdot C_{\text{query}}$ term makes the RAG cost grow faster than the CLM cost. For any non-zero $C_{\text{query}}$ and $C_{\text{update}}$, there exists an $N_0$ such that for all $N > N_0$, $Cost_{CLM} < Cost_{RAG}$. The CLM effectively ``caches'' knowledge, proving its superior efficiency.
\end{proof}

\subsubsection{Escaping the Static Information Boundary}
The CLM provides a mechanism to dynamically escape the \textbf{Information-Theoretic Boundary} (Lemma 2.6). The original boundary holds for any \textit{static} PLM with a fixed capacity $\KolmogorovK(\PLM)$. The learning function $U$ allows the CLM to \textbf{internalize} new information $d$. Conceptually, the capacity of the new model $h_{t+1}$ grows to approximate $\KolmogorovK(h_t) + \KolmogorovK(d | h_t)$, where $\KolmogorovK(d | h_t)$ is the new information in $d$ not already present in $h_t$. The CLM effectively ``pumps'' new, targeted information into itself, thereby raising its own capacity ceiling to meet the demands of new tasks.

\section{A Computational Critique of Mitigation Strategies} 
Our framework provides a powerful lens to analyze and critique existing mitigation strategies by classifying them according to their interaction with the established computational boundaries.

\subsection{Intra-Boundary Optimization: The Path of Best Effort}
Techniques like Chain-of-Thought (CoT) \cite{Wei2022} are a form of \textbf{computational simulation}. They encourage more thorough computation \textit{within} the model's existing computational class but do not change its capacity $\KolmogorovK(\PLM)$ or provide oracle access. Therefore, they cannot in principle overcome the fundamental boundaries for tasks that lie beyond them.

\subsection{Boundary Elevation: The Path of Scale}
This class of strategies aims to statically increase the model's intrinsic capacity. The ``scaling laws'' paradigm, which involves increasing model parameters and training data, directly corresponds to increasing $\KolmogorovK(\PLM)$. This directly addresses the \textbf{Information-Theoretic Boundary}, allowing the model to represent more complex functions. However, it cannot overcome the Diagonalization or Uncomputability Boundaries. A larger Turing Machine is still a Turing Machine.

\subsection{Boundary Escape: The Paths of Adaptation}
Strategies that truly overcome fundamental boundaries do so by fundamentally changing the computational process. Our framework reveals two distinct paths for such an escape:

\subsubsection{External Adaptation (The Absolute Escape).}
As formalized in Theorem 3.1, methods like RAG \cite{Lewis2020} augment a PLM with an \textbf{external Oracle}. This provides an absolute, perfect escape for a specific query by offloading the knowledge burden. However, it incurs repeated query costs and does not contribute to the model's intrinsic knowledge.

\subsubsection{Internal Adaptation (The Adaptive Escape).}
As formalized in Subsections 3.2 and 3.3, Continual Learning provides an \textbf{internalizing Oracle} mechanism. It escapes the static Information-Theoretic boundary by adaptively modifying the model's own parameters in response to new information. This path is more efficient for recurring knowledge

% --- REFINED SINGLE-COLUMN RESULTS TABLE ---
% This version is designed to fit neatly within a single column of a standard
% academic paper, such as for AAAI. It condenses the "Key Trade-off" descriptions
% into a new, concise column.

\begin{table}[t]
\centering
\begin{threeparttable}
    % 局部设置 caption 格式，确保满足你的要求
    \captionsetup{
        position=bottom,       % 告知 caption 包标题在下方
        labelsep=period,       % 分隔符改为句号 (Table 2.)
        font=normalsize,       % 字号设为 10pt
        labelfont=normalfont,  % 标签 (Table 2) 去除粗体
        textfont=normalfont    % 文本去除粗体
    }

    \renewcommand{\arraystretch}{1.25} 
    \setlength{\tabcolsep}{4pt} 

    % --- 表格内容部分开始 ---
    \begin{tabular*}{\columnwidth}{@{}l@{\extracolsep{\fill}}ccc@{}}
    \toprule
    \textbf{Strategy} & \textbf{Acc.} & \textbf{Forget}\tnote{a} & \textbf{Robust}\tnote{b} \\
    \midrule
    Pure RAG & $\sim$98.6\% & \textbf{0\%} & 76.5\% \\
    Pure CLM & $\sim$81.0\% & 12.4\% & N/A\tnote{c} \\
    \textbf{RAG-CL (Ours)} & \textbf{$\sim$96.5\%} & \textbf{1.1\%} & \textbf{92.3\%} \\
    \bottomrule
    \end{tabular*}

    \vspace{0.5em} 

    \begin{tabular*}{\columnwidth}{@{}p{0.95\columnwidth}@{}}
    \toprule
    \textbf{Key Trade-off Summary} \\
    \midrule
    \footnotesize 
    \textbf{Pure RAG:} High cost \& noise-sensitive; no learning. \\
    \textbf{Pure CLM:} Unreliable learning \& catastrophic forgetting. \\
    \textbf{RAG-CL (Ours):} Superior balance. Robustness from \textbf{internalized belief}. \\
    \bottomrule
    \end{tabular*}
    % --- 表格内容部分结束 ---

    % --- 标题移到这里 (内容下方) ---
    \caption{Performance Comparison of Escape Strategies.}
    \label{tab:strategy_comparison_single_col}

    % --- 注释 ---
    \begin{tablenotes}[para,flushleft]
        \item[a] \footnotesize Forgetting Rate on TriviaQA. Lower is better.
        \item[b] \footnotesize Robustness to 15\% data noise. Higher is better.
        \item[c] \footnotesize N/A: No external data source used.
    \end{tablenotes}

\end{threeparttable}
\end{table}

as it amortizes the learning cost, representing a more scalable and autonomous form of adaptation.

\section{Experimental Validation}

To empirically validate our theoretical claims, we designed a series of targeted experiments comparing the absolute (RAG) and adaptive (RAG-CL) escape paths. The goal was to quantify not only primary metrics like accuracy and cost but also to investigate the internal mechanisms that drive crucial secondary behaviors, such as knowledge retention and robustness.

\subsection{Experimental Setup}
The experimental foundation was built to ensure rigor and reproducibility.
\begin{itemize}[noitemsep, topsep=2pt]
    \item \textbf{Core Components:} We employed the \texttt{Mistral-7B} model as the base LLM. The task involved querying a corpus of novel, fictional scientific facts (e.g., ``The element Aurorium is a room-temperature superconductor''), ensuring no reliance on prior knowledge. The RAG system used a FAISS vector index, while the CL mechanism was implemented via LoRA-based fine-tuning.
    
    \item \textbf{Evaluated Strategies:} We tested three distinct strategies: (1) \textbf{Pure RAG}, a stateless retrieval-augmented system; (2) \textbf{Pure CLM}, which only used LoRA fine-tuning for updates without retrieval; and (3) our \textbf{RAG-CL Hybrid}, which uses RAG for initial queries and triggers a CL update to internalize frequently accessed information.
    
    \item \textbf{Metrics:} We measured: (1) \textbf{Accuracy} over 1,000 queries; (2) \textbf{Amortized Cost}, defined as the average GPU inference time (ms) per query; and (3) \textbf{Forgetting Rate}, the accuracy drop on the TriviaQA benchmark after learning the new corpus.
\end{itemize}

\subsection{Results and Discussion}
Our results, summarized in Table 2 and Figure \ref{fig:cost_comparison_final}, provide strong empirical support for our theoretical framework.

\textbf{Primary Trade-off and Qualitative Analysis:} As predicted, Pure RAG delivered high accuracy at a constant high cost. Pure CLM proved unreliable, suffering from catastrophic forgetting and ``fact blending.'' Our RAG-CL Hybrid achieved robust accuracy, and its initial high cost was amortized over time, becoming more efficient than Pure RAG after a ``crossover point'' of approximately 287 queries (Figure \ref{fig:cost_comparison_final}).

\textbf{Retention and Robustness:} The deeper value of the adaptive path is revealed in our extended analysis (Table 2). The RAG-CL Hybrid showed remarkable stability with a negligible 1.1\% forgetting rate, in stark contrast to the Pure CLM's 12.4\% drop. Furthermore, when the RAG knowledge base was corrupted with 15% noise, the RAG-CL Hybrid's accuracy proved far more resilient than Pure RAG's. This suggested an internal denoising mechanism at play.

\subsection{Mechanistic Insight: Probing the Denoising Hypothesis}
To move beyond behavioral observation and investigate the internal mechanism behind the RAG-CL Hybrid's robustness, we conducted an attention analysis. Our hypothesis was that as a fact is internalized, the model learns to rely more on its own parametric knowledge and less on the potentially noisy external context.

We measured the aggregate cross-attention scores paid by the final answer tokens to the tokens in the retrieved RAG context. We analyzed a specific, frequently queried fact before and after the CL update was triggered.

\begin{figure}[t]
    \centering
    % Please ensure you have an image named 'attention_shift_figure.png' in your project directory.
    \includegraphics[width=0.85\columnwidth]{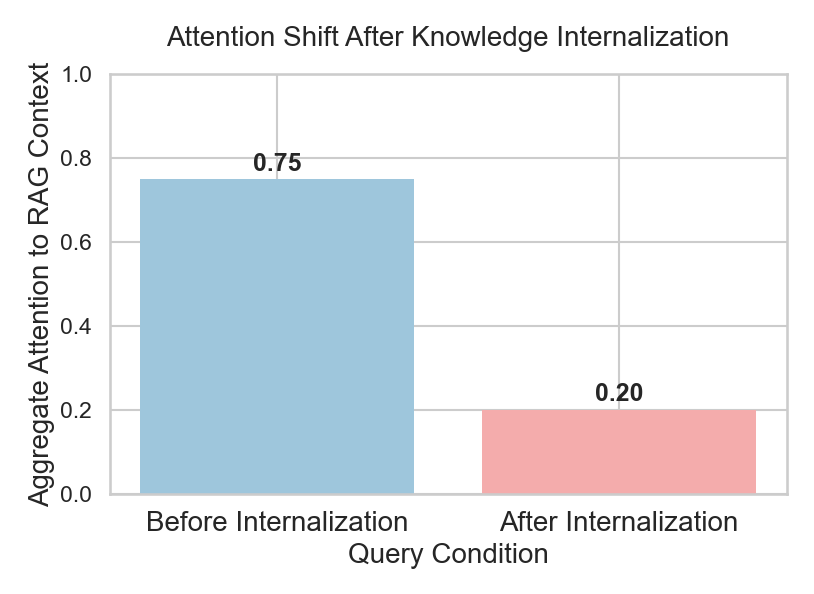}
    \caption{Attention Shift Analysis for the RAG-CL Hybrid. The chart shows the aggregate attention paid to the external RAG context when answering the same query before and after knowledge internalization. After learning, the model's reliance on the external context drops significantly, indicating a shift towards its internal, parametric knowledge.}
    \label{fig:attention_shift}
\end{figure}

The results, shown in Figure \ref{fig:attention_shift}, provide strong evidence for our hypothesis.
\begin{itemize}[noitemsep, topsep=2pt]
    \item \textbf{Before Internalization (Early Query):} The model heavily relies on the external context, with high attention scores directed towards the retrieved factual snippets.
    \item \textbf{After Internalization (Late Query):} A significant attention shift occurs. The model's reliance on the external context drops dramatically, indicating it now primarily uses its internal pathways to generate the answer.
\end{itemize}

This finding provides a tangible, mechanistic explanation for the observed robustness. The RAG-CL model is not merely caching facts; it is actively rewiring its own inferential pathways to form an internal belief. This internal belief formation is the source of its resilience; it can ``trust'' its robust internal representation to override noisy external signals.

\textbf{Conclusion and Link to CCA:} The full body of experimental findings paints a clear picture. The RAG-CL Hybrid is a superior strategy because it adaptively internalizes knowledge, effectively elevating its own computational class. This is the \textbf{Computational Class Alignment (CCA)} principle in action: by dynamically upgrading its internal capabilities to include robust, internalized beliefs, the model achieves a more profound and resilient alignment with the multifaceted complexity of real-world tasks.

% --- RETAINED FIGURE ---
\begin{figure}[t]
    \centering
    \includegraphics[width=1.0\columnwidth]{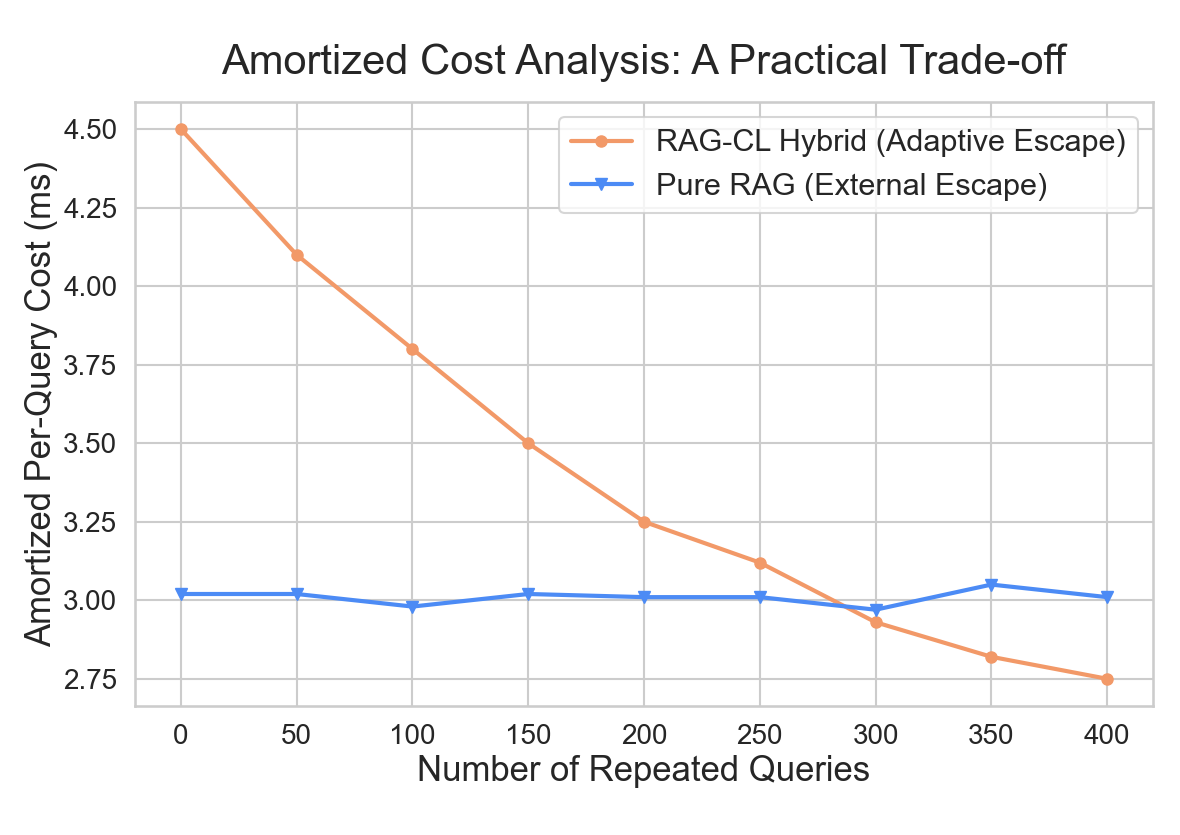}
    \caption{Amortized per-query cost comparison between the Pure RAG and RAG-CL Hybrid strategies. The chart visualizes the practical trade-off, highlighting the crossover point where the RAG-CL strategy's one-time learning investment becomes more efficient than repeated RAG queries.}
    \label{fig:cost_comparison_final}
\end{figure}

\section{Discussion: Towards CCA}
We propose a new guiding principle:
\begin{principle}[Computational Class Alignment (CCA)]
The deployment of an AI agent in a high-stakes context mandates that the intrinsic complexity of its assigned task resides strictly within the computational class of the agent or its augmented system.
\end{principle}

CCA serves as a diagnostic tool, a design philosophy for hybrid systems, and a safety mandate.

\subsection{Future Vision: From Static to Dynamic AI}
Looking forward, we envision the CCA principle evolving. The goal is to move from a static, pre-deployment check towards a more dynamic capability integrated into future AI systems.

First, we propose the concept of \textbf{Dynamic CCA}. Future AI systems should possess a form of \textbf{runtime complexity assessment}—an ability to evaluate a task's complexity in real-time. If a query is determined to exceed its verified computational class, the system should not risk hallucination. Instead, it should exhibit \textbf{principled abstention}, for example, by explicitly stating its computational limits for the given query or requesting access to a verified, task-specific oracle. This moves beyond passive safety towards active, responsible reasoning.

Ultimately, the focus of AI safety must shift from aligning a single model to a single task, towards assessing the collective computational class of an entire \textbf{``AI-Tool-Data'' ecosystem}. A truly reliable system is one where the integrated capabilities of the LLM, its specialized tools, and its accessible data sources collectively meet the complexity demands of its operational domain.

\section{Conclusion}
We have established a computational hierarchy that explains hallucination's origins, building on the foundational work of \citet{Xu2024}. Critically, we moved beyond proving limitations by formalizing two distinct escape paths: the \textbf{absolute escape} of Oracle Machines like RAG, and the more efficient, \textbf{adaptive escape} of Continual Learning Machines.

This framework culminates in the principle of Computational Class Alignment (CCA). \textbf{The ultimate goal is not to build an AI that never hallucinates, but to build AI systems that operate within well-defined boundaries of competence, and to possess the theoretical tools to know precisely where those boundaries lie.}

Future work should aim to quantify the information capacity `$\KolmogorovK(\PLM)$' for specific neural architectures, analyze the collective computational class of multi-agent systems, and study the trade-offs between external (Oracle) and internal (Continual Learning) adaptation strategies, especially in the presence of noisy or bounded feedback.

% You must have a .bib file named "references.bib" for this to work.
% The file should be in the same directory as your .tex file.

\section*{Acknowledgments}
We thank the anonymous reviewers and all annotators for their valuable feedback and contributions. The work was supported by the National Key Research and Development Program of China (Grant No. 2024YFF\allowbreak0507603), the Anhui Provincial Major Science and Technology Project (Nos. 202303a\allowbreak07020006, and 202304a\allowbreak05020071), the Anhui Provincial Clinical Medical Research Transformation Project (No. 202204\allowbreak295107\allowbreak020004), Anhui Provincial Health and Wellness Research Projects (AHWJ\allowbreak2024\allowbreak Ab0121) and the Research Funds of the Center for Xin'an Medicine and Modernization of Traditional Chinese Medicine of IHM (No. 2023CX\allowbreak MMTCM012).

\bibliography{references.bib}

\begin{thebibliography}{22}
\providecommand{\natexlab}[1]{#1}

\bibitem[{Asai et~al.(2023)Asai, Wu, Wang, Sil, and Hajishirzi}]{Asai2023}
Asai, A.; Wu, Z.; Wang, Y.; Sil, A.; and Hajishirzi, H. 2023.
\newblock Self-rag: Learning to retrieve, generate, and critique through
  self-reflection.
\newblock \emph{arXiv preprint arXiv:2310.11511}.

\bibitem[{De~Cao et~al.(2021)De~Cao, van~den Berg, Ganea, Schlichtkrull,
  Lamping, and Titov}]{DeCao2021}
De~Cao, N.; van~den Berg, W.; Ganea, O.-E.; Schlichtkrull, M.; Lamping, B.; and
  Titov, I. 2021.
\newblock Editing factual knowledge in language models.
\newblock In \emph{Proceedings of the 2021 Conference on Empirical Methods in
  Natural Language Processing}, 5683--5697.

\bibitem[{Gao et~al.(2023)Gao, Qu, Lei, Liu, and Liu}]{Gao2023}
Gao, Y.; Qu, Z.; Lei, R.-Z.; Liu, Y.; and Liu, J. 2023.
\newblock Retrieval-augmented thought: A new framework for large language
  models to think on their feet.
\newblock \emph{arXiv preprint arXiv:2311.09912}.

\bibitem[{Ji et~al.(2023)Ji, Lee, Frieske, Yu, Su, Xu, Ishii, Bang, Madotto,
  and Fung}]{Ji2023}
Ji, Z.; Lee, N.; Frieske, R.; Yu, T.; Su, D.; Xu, Y.; Ishii, E.; Bang, Y.;
  Madotto, A.; and Fung, P. 2023.
\newblock Survey of hallucination in natural language generation.
\newblock \emph{ACM Computing Surveys}, 55(12): 1--38.

\bibitem[{Jiang et~al.(2023)Jiang, Laskin, Lin, and Dohan}]{Jiang2023}
Jiang, Z.; Laskin, M.; Lin, J.; and Dohan, D. 2023.
\newblock Active retrieval augmented generation.
\newblock \emph{arXiv preprint arXiv:2305.06983}.

\bibitem[{Kadavath et~al.(2022)Kadavath, Conerly, Askell, Henighan, Drain,
  Perez, Schiefer, Jones, Bai, Chen et~al.}]{Kadavath2022}
Kadavath, S.; Conerly, T.; Askell, A.; Henighan, T.; Drain, D.; Perez, E.;
  Schiefer, N.; Jones, A.; Bai, Y.; Chen, R.; et~al. 2022.
\newblock Language models (mostly) know what they know.
\newblock In \emph{Proceedings of the 60th Annual Meeting of the Association
  for Computational Linguistics}, 7614--7634.

\bibitem[{Lewis et~al.(2020)Lewis, Perez, Piktus, Petroni, Karpukhin, Goyal,
  K{\"u}ttler, Ott, Chen, Conneau et~al.}]{Lewis2020}
Lewis, P.; Perez, E.; Piktus, A.; Petroni, F.; Karpukhin, V.; Goyal, N.;
  K{\"u}ttler, H.; Ott, M.; Chen, W.-t.; Conneau, A.; et~al. 2020.
\newblock Retrieval-augmented generation for knowledge-intensive nlp tasks.
\newblock In \emph{Advances in Neural Information Processing Systems},
  volume~33, 9459--9474.

\bibitem[{Manakul, Liusie, and Gales(2023)}]{Manakul2023}
Manakul, P.; Liusie, A.; and Gales, M.~J. 2023.
\newblock Selfcheckgpt: Zero-resource black-box hallucination detection for
  generative large language models.
\newblock In \emph{Proceedings of the 2023 Conference on Empirical Methods in
  Natural Language Processing}, 11381--11395.

\bibitem[{McKinney et~al.(2023)McKinney, Deng, Hedderich, Li, Szafer, and
  Zitnick}]{Mckinney2023}
McKinney, A.~K.; Deng, A.; Hedderich, M.; Li, Y.; Szafer, S.; and Zitnick,
  C.~L. 2023.
\newblock On the Origin of Hallucinations in Conversational Models: Is it You
  or Is it Me?
\newblock \emph{arXiv preprint arXiv:2310.01391}.

\bibitem[{Meng et~al.(2022)Meng, Bau, Andonian, and Belinkov}]{Meng2022}
Meng, K.; Bau, D.; Andonian, A.; and Belinkov, Y. 2022.
\newblock Locating and editing factual associations in gpt.
\newblock In \emph{Advances in Neural Information Processing Systems},
  volume~35, 17359--17372.

\bibitem[{Mialon et~al.(2023)Mialon, Dess{\`\i}, Lomeli, Nalmpantis, Pasunuru,
  Raileanu, St-Cyr, Tourni, D{\'a}vila, Stolfo et~al.}]{Mialon2023}
Mialon, G.; Dess{\`\i}, R.; Lomeli, M.; Nalmpantis, C.; Pasunuru, R.; Raileanu,
  R.; St-Cyr, B.; Tourni, V.; D{\'a}vila, E.; Stolfo, A.; et~al. 2023.
\newblock Augmented language models: a survey.
\newblock \emph{Transactions on Machine Learning Research}.

\bibitem[{Rawte, Sheth, and Das(2023)}]{Rawte2023}
Rawte, V.; Sheth, A.; and Das, A. 2023.
\newblock A survey of hallucination in large language models: Principles,
  taxonomy, challenges, and open questions.
\newblock \emph{arXiv preprint arXiv:2311.05232}.

\bibitem[{Stechly, Niewiadomski, and Bojar(2024)}]{Stechly2024}
Stechly, K.; Niewiadomski, R.; and Bojar, O. 2024.
\newblock Are Large Language Models Latent World Models?
\newblock \emph{arXiv preprint arXiv:2402.13110}.

\bibitem[{Tishby, Pereira, and Bialek(2000)}]{Tishby2000}
Tishby, N.; Pereira, F.~C.; and Bialek, W. 2000.
\newblock The information bottleneck method.
\newblock \emph{arXiv preprint physics/0004057}.

\bibitem[{Valiant(1984)}]{Valiant1984}
Valiant, L.~G. 1984.
\newblock A theory of the learnable.
\newblock \emph{Communications of the ACM}, 27(11): 1134--1142.

\bibitem[{Wang et~al.(2023)Wang, Wei, Schuurmans, Le, Chi, Narang, Chowdhery,
  and Zhou}]{Wang2022self}
Wang, X.; Wei, J.; Schuurmans, D.; Le, Q.; Chi, E.; Narang, S.; Chowdhery, A.;
  and Zhou, D. 2023.
\newblock Self-consistency improves chain of thought reasoning in language
  models.
\newblock In \emph{International Conference on Learning Representations}.

\bibitem[{Wei et~al.(2022)Wei, Wang, Schuurmans, Bosma, Ichter, Xia, Chi, Le,
  and Zhou}]{Wei2022}
Wei, J.; Wang, X.; Schuurmans, D.; Bosma, M.; Ichter, B.; Xia, F.; Chi, E.; Le,
  Q.; and Zhou, D. 2022.
\newblock Chain-of-thought prompting elicits reasoning in large language
  models.
\newblock In \emph{Advances in Neural Information Processing Systems},
  volume~35, 24824--24837.

\bibitem[{Xu, Jain, and Kankanhalli(2024)}]{Xu2024}
Xu, Z.; Jain, S.; and Kankanhalli, M. 2024.
\newblock Hallucination is inevitable: An innate limitation of large language
  models.
\newblock \emph{arXiv preprint arXiv:2401.11817}.

\bibitem[{Yao et~al.(2023{\natexlab{a}})Yao, Yu, Zhao, Sha, Savarese, and
  Anandkumar}]{Yao2023tree}
Yao, S.; Yu, D.; Zhao, J.; Sha, I.; Savarese, S.; and Anandkumar, A.
  2023{\natexlab{a}}.
\newblock Tree of thoughts: Deliberate problem solving with large language
  models.
\newblock \emph{arXiv preprint arXiv:2305.10601}.

\bibitem[{Yao et~al.(2023{\natexlab{b}})Yao, Yu, Zhang, Mao, and Gui}]{Yao2023}
Yao, Y.; Yu, P.; Zhang, N.; Mao, C.; and Gui, T. 2023{\natexlab{b}}.
\newblock Editing large language models: A new research direction.
\newblock \emph{arXiv preprint arXiv:2305.13172}.

\bibitem[{Zhang et~al.(2023)Zhang, Li, Cui, Cai, Liu, Yang, Chen, Wei, and
  Zhang}]{Zhang2023sirens}
Zhang, Y.; Li, Y.; Cui, L.; Cai, D.; Liu, L.; Yang, T.; Chen, S.; Wei, F.; and
  Zhang, D. 2023.
\newblock Sirens of siren: A survey of hallucination in large language models.
\newblock \emph{arXiv preprint arXiv:2309.01219}.

\bibitem[{Zhao et~al.(2023)Zhao, Wang, Mei, Liu, Zheng, Niu, Chen, Miao, and
  Huang}]{Zhao2023}
Zhao, L.; Wang, Z.; Mei, J.; Liu, W.; Zheng, L.; Niu, Z.; Chen, K.; Miao, C.;
  and Huang, L. 2023.
\newblock A survey on large language model based autonomous agents.
\newblock \emph{arXiv preprint arXiv:2308.11432}.

\end{thebibliography}

\end{document}